\documentclass[letterpaper, 10 pt, conference]{ieeeconf}  % Comment this line out
\IEEEoverridecommandlockouts                              % This command is only
\usepackage[pdftex]{graphicx}
\usepackage[cmex10]{amsmath}
\usepackage{amssymb}
\usepackage{import}
\usepackage{bm}
\usepackage{color}
\usepackage{units}
\usepackage{url}
\usepackage{float}
\usepackage{epstopdf}
\usepackage{stmaryrd}
\usepackage{units}
\usepackage{siunitx}
\usepackage{outlines}
\usepackage{boldline}
\usepackage[flushleft]{threeparttable}

\usepackage[doi=false,isbn=false,defernumbers=false,style=ieee,backend=biber,maxbibnames=1000]{biblatex}

\usepackage[utf8]{inputenc}
\usepackage{mathtools}
\usepackage{color,soul}
\usepackage{ntheorem}

\newcommand{\mnoshow}[1]{}

\newcommand{\figref}[1]{Fig.~\ref{#1}}
\newcommand{\secref}[1]{Section~\ref{#1}}

\newcommand*{\rect}{\textnormal{rect}}

\newcommand{\diag}{\textnormal{diag}}

\DeclarePairedDelimiter\floor{\lfloor}{\rfloor}

\DeclareMathOperator*{\argmin}{arg\,min}

\newtheorem{theorem}{Theorem}
\newtheorem{defn}{Definition}

\title{\LARGE \bf
In Proximity of ReLU DNN, PWA Function, and Explicit MPC}

\author{Saman Fahandezh-Saadi, Masayoshi Tomizuka
\thanks{The authors are with the Department of Mechanical Engineering, at the University of California, Berkeley. {\tt \{samanfahandej, tomizuka\}@berkeley.edu}}%
}

\begin{document}

\maketitle
\thispagestyle{empty}
\pagestyle{empty}

%%%%%%%%%%%%%%%%%%%%%%%%%%%%%%%%%%%%%%%%%%%%%%%%%%%%%%%%%%%%%%%%%%%%%%%%%%%%%%%%
\begin{abstract}
Rectifier (ReLU) deep neural networks (DNN) and their connection with piecewise affine (PWA) functions is analyzed. The paper is an effort to find and study the possibility of representing explicit state feedback policy of model predictive control (MPC) as a ReLU DNN, and vice versa. The complexity and architecture of DNN has been examined through some theorems and discussions. An approximate method has been developed for identification of input-space in ReLU net which results a PWA function over polyhedral regions. Also, inverse multiparametric linear or quadratic programs (mp-LP or mp-QP) has been studied which deals with reconstruction of constraints and cost function given a PWA function.
\end{abstract}

\section{Introduction}
In recent years, deep neural networks (DNN) has had tremendous success in computer vision, speech recognition, and other areas of machine learning \cite{NIPS2012_4824, goodfellow13, go_game, human_learning}. Despite all these unprecedented performances in learning tasks, a theoretical understanding of DNN's architecture, features, and properties is still unexplored. Also, all of these successes are related to the supervised learning and are concerned mostly with function fitting (e.g. classification, function approximation, and regression). In contrast, in reinforcement learning (RL), the concept of feedback makes it hard to study in theory since the statistical properties are dynamic/changing, and also they are hard to train in practice. Another shortcoming of DNN in RL is the absence of theoretical guarantees regarding stability, robustness, and convergence. All these issues need a great deal of consideration.

On the other hand, model predictive control is a powerful tool for control and decision making in robotics and other safety-concerned applications due to its adaptability, robustness, and stability-safety guarantees. In specific, \emph{explicit} MPC allows us to pre-compute the optimal control policy $u_t^* = f(x(t))$ as a function of current state $x(t)$, and deploy it on-line in real-time. This prevents the issue of solving optimization problem in real-time on embedded systems which are typically limited with regard to memory capacity and computation power. But deployment of an explicit MPC suffers from increasing number of regions which grows exponentially (in the worst case) with the number of constraints \cite{ExplicitmpcSurvey, explicit}. This demands significant amount of storage and computational complexity.

Several attempts have been made to address those shortcomings in \emph{explicit} MPC \cite{approxMPC1, approxMPC2, approxMPC3, approxMPC4, approxMPC5, uniquePoly}. But in contrast, regarding deep reinforcement learning, all the attempts were mainly focused on empirical results, and analyzing its architecture only to appear in literature in very recent years. Here we focus more to mention some of these new findings regarding the DNN. Authors in \cite{BoundingCounting} investigate the complexity of DNN by studying the number of polytopic regions that they can attain. The paper also provides a tighter upper-bound (compared to previous bounds \cite{Montufar2017}) on the maximal number of regions that can be partitioned by a ReLU DNN. The paper \cite{ClassificationRegionsFolding} discusses the geometric properties of DNN for classification and how to improve the robustness of such DNN to perturbation by analyzing those properties. In \cite{safeRL} authors present a method that adds stability guarantee to the deep gradient descent algorithm.
% 1) Complexity of ReLU net can be measured by the number of linear pieces of the function it computes. 2) Folding: ReLU nets successively “fold” the input domain. 3) Lots of researches are focused on bottleneck effect (for example when $n_l \ll n_{l+1}$), reducing the ability of DNN

\begin{figure}[!t]
    \centering
    \includegraphics[width=\linewidth]{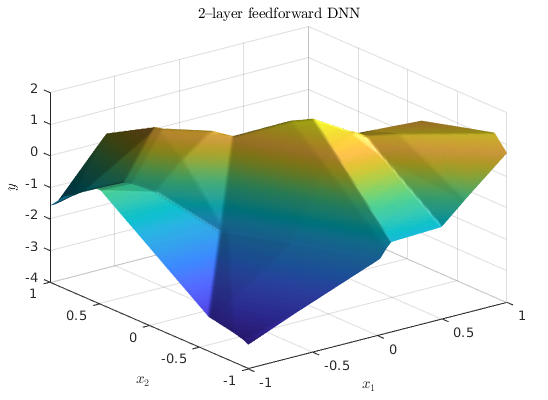}
    \caption{This example shows the level of complexity that a NN can represent. The plot shows how a $2$-layer NN maps input-space $x\in\mathbb{R}^2$ to the output-space $y\in\mathbb{R}$ with $7$ ReLU activation units in each layer (total of $84$ parameters). The network creates a complex continuous PWA function which can be a close approximation of a highly nonlinear function.}
    \label{fig:NN_exp}
\end{figure}

Although these two specific areas of research (deep RL and MPC) have strong connection in (adaptive) optimal control theory \cite{RLMPC}, but from mathematical point of view there is another link between these two: Both ReLU DNN and solution to the mp-LP or mp-QP in \emph{explicit} MPC represent a PWA function on polyhedra. This gives a great amount of motivation to investigate the possibility of reconstructing one from the other in order to benefit from advantages in both approaches.

% A general formulation of multiparametric programming can be written as
% \begin{equation}\label{eq:mpp}
% \begin{aligned}
%     f(x)\ \in\ & \argmin_z \ \ J(z,x) \\
%     & \textnormal{s.t.} \quad (z,x) \in \Omega,
% \end{aligned}
% \end{equation}
% where $z\in \mathbb{R}^n$, $x\in \mathbb{R}^m$ are decision variables, parameters respectively, and $\Omega \in \mathbb{R}^{n+m}$. The problem is called linear (mp-LP) when the objective is linear, $J(z,x) = c^\top z$ and quadratic (mp-QP) when the objective is quadratic, $J(z,x) = \frac{1}{2}z^\top Q z$. The constraints are $\Omega := Gz \leq w + Sx$ and we assume that \eqref{eq:mpp} is feasible on polyhedral regions. The explicit solution to both mp-LP and mp-QP $f(x)$ is PWA function on polytope regions \cite{borrelli2017predictive}.We are interested in its inverse, i.e. given a solution function $f$ we wish to construct a convex set $\Omega$ and cost function $J$ such that $f$ satisfies \eqref{eq:mpp}. It is known that such a mPCP always exists for any continuous function $f$ equipped with a (possibly artificially chosen) value function \cite{mpcp}. The authors showed that every continuous nonlinear control law f can be characterized as the solution to some parametric convex program (PCP) assuming that the parametric value function is also known.
Since the presumption concerning DNN that they have tens of thousands of parameters (weights and biases) seems reasonable for vision or language applications, but in fact a DNN can represent a very complex function with much less number of parameters. This is a compelling property when we are dealing with representing a control policy as a DNN. As an example, \figref{fig:NN_exp} shows a $2$--layer ReLU network with just $84$ parameters chosen randomly. The plot shows how a very small size network can subdivide the input-space to many polytopes and different affine policy pieces over each region.

% In the following, we first provide mathematical definition of ReLU DNN and a review on some recently developed structural properties of neural nets in the form of some theorems. Then, we present a brief overview of existing theorems that represent their connection to PWA functions, and discuss challenges which prevent to have an explicit association. Next, we present a sample-based method in order to identify the underlying PWA function that a ReLU DNN can represent. Finally, we provide a numerical example that examine a simple network and its equivalent PWA function.

In the following, we first provide mathematical definition of ReLU DNN and its structural properties in \secref{secNN}. Then in \secref{secBOTH}, we present a brief overview of existing theorems that represent connection between ReLU nets and PWA functions, and discuss challenges which prevent us to have an explicit association. Also we present a sample-based method in order to identify the underlying PWA function that a ReLU DNN can represent. Finally in \secref{secExpValidation}, we provide a numerical example that examine a simple network and its equivalent PWA function.

\section{Preliminaries and Problem Formulation}
\label{secNN}
\noindent
In this section we define feedforward ReLU DNN and discuss some properties of these models and their ability to map input-space to the complex family of PWA functions.
\subsection{Notation and Definitions}
\begin{defn}\label{def:nn_notations}
A rectifier (ReLU) feedforward network is a layered neural network with $L\in \mathbb{N}$ hidden layers (\emph{depth} of the net) with input and output dimensions $n_0, n_{L+1}\in \mathbb{N}$ respectively. Each hidden layer $l$ is composed of an affine transformation $f_{l}: \mathbb{R}^{n_l} \rightarrow \mathbb{R}^{n_{l+1}}$ followed by a rectifier activation function $\rect(x): x \mapsto \max(x, 0)$
\begin{align*}
    &f_l= W_l h_{l-1} + b_l\\
    &h_l = \rect(f_l) = \max \{f_l,\ 0\},
\end{align*}
where the max is an element-wise function, $W_l\in \mathbb{R}^{n_{l}\times n_{l-1}}$, $b_l\in \mathbb{R}^{n_l}$, $f_l\in \mathbb{R}^{n_l}$, $h_l\in \mathbb{R}^{n_l}$, and $h_0\in \mathbb{R}^{n_0}$ is defined as the input to the network. We call $f_l$ \emph{pre-activation} and $h_l$ \emph{post-activation} functions at hidden layer $l$. The output layer is just a linear transformation $W^{(L+1)}$ and does not count as part of the hidden layers. Finally, any ReLU net with $L > 1$ layers is called \emph{$L$--layer DNN} and can be represented as a function $f:\mathbb{R}^{n_0} \rightarrow \mathbb{R}^{n_{L+1}}$
\begin{align*}
    f = W_{L+1} \circ h_L \circ f_L \circ \hdots \circ h_{2} \circ f_{2} \circ h_1 \circ f_1 
\end{align*}
where $\circ$ denotes function decomposition.
\end{defn}
\begin{figure}[t]
    \centering
    \includegraphics[width=\linewidth]{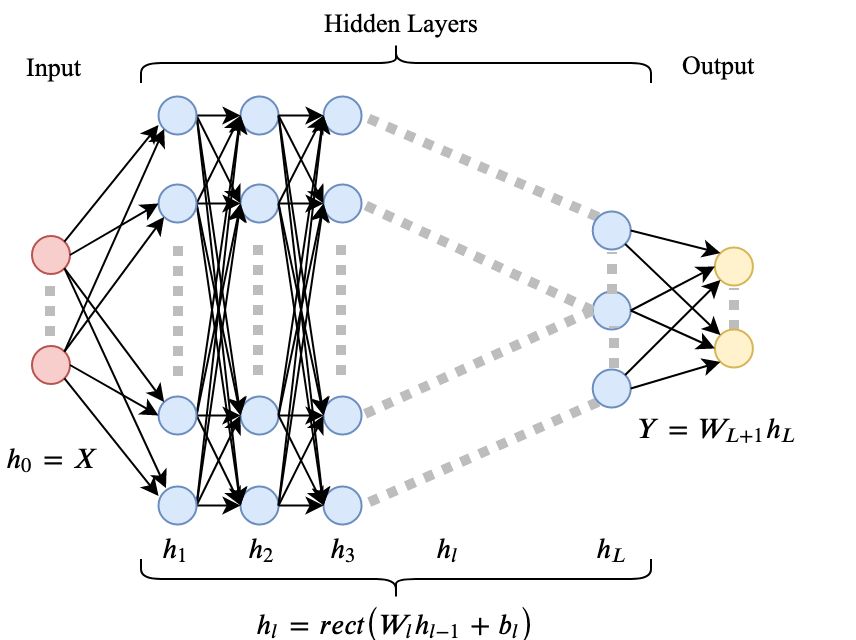}
    \caption{Illustration of $L$--layer ReLU DNN $f:\mathbb{R}^{n_0}\rightarrow \mathbb{R}^{n_{L+1}}$. Depending on application and complexity of function to be approximated, DNN can have an arbitrary number of layers (i.e. depth) and activation units in each layer (i.e width). Note that DNNs are recognized just by the number of \emph{hidden layers}. Also as seen, the output layer is just a linear transformation of last hidden layer without activation mapping.}
    \label{fig:NNillustration}
\end{figure}
\begin{defn}\label{def:relu}
Every layer $l\in \{1,2,\hdots,L\}$ of a ReLU DNN has $n_l$ \emph{activation units} which is called \emph{width} of the layer. Each activation unit receives the rectified weighted sum of the previous post-activation values $h_{l-1}$ plus a bias. The $j$th activation unit in layer $l$ is denoted by $h_{l, j}\in \mathbb{R}$
\begin{align*}
    f_{l, j} &= W^\top_{l,j} h_{l-1} + b_{l,j} \\
    h_{l, j} &= \max\{f_{l, j}, 0\},
\end{align*}
where $W^\top_{l,j}$ and $b_{l,j}$ are the $j$th row and the $j$th element of matrix $W_{l}$ and vector $b_{l}$, respectively.
\end{defn}

An illustration of a $L$--layer ReLU DNN is shown in \figref{fig:NNillustration}. Each blue circle in the figure represent an activation unit. Depending on the structure of DNN it can have any width size for each hidden layer. The total number of parameters for each DNN $\theta = \{W_{1:L+1}, b_{1:L}\}$ can be a basis to compare different architectures by varying depths and widths.
\subsection{ReLU DNN Expressiveness}
Despite the DNN's empirical successes, some fundamental questions about how and why these results are achieved is absent in literature. \emph{Neural net expressivity} is a subject that tries to answer some of these questions such as how the depth, width, and the type of layers impact the function that the network represents, and also how these properties affect its performance. Here we try to provide some of these findings. First we present a set of theorems that deal with these types of questions.  
% \begin{theorem}\label{twowaythm}
% Every ReLU DNN $f: \mathbb{R}^{n_0}\rightarrow \mathbb{R}$ represents a piecewise affine function, and every piecewise affine function $f: \mathbb{R}^{n_0}\rightarrow \mathbb{R}$ can be represented by a ReLU DNN with at most $\ceil{\log_2(n_0+1)} + 1$ layers.  
% \end{theorem}
% \begin{proof}    
% See Theorem 2.1 in \cite{arora2018understanding}.
% \end{proof}

First, since the post-activation $h(s) := \max\{s, 0\}$ is itself a PWA function and also the structure of ReLU networks is a series of composition of affine and post-activation functions, therefore the result is a PWA function that is defined over the regions of the input-space. This has been stated in the following theorem.
\begin{theorem}
Given a neural network with ReLU activation, the input-space is partitioned into convex polytopes. 
\end{theorem}
\begin{proof}
The complete proof can be found in \cite{OnTheExpressivePower}. But as sketch of proof, consider the first layer $l=1$; each pre-activation function establishes a hyperplane on the input-space since $f_{1,j} = W_{1,j}x+b_{1,j} = 0$. All such hyperplanes associated to each unit provide a \emph{hyperplane arrangement} which partitions the input-space into polytopes. By induction, it can be shown that this is true for all other layers in DNN. \figref{fig:hyperplanes} illustrates the theorem for a $2$-layer DNN.
\end{proof}

Another important property of ReLU networks is the number of polytopic regions that they can realize on their input-spaces. This helps on two fronts: 1) To understand the complexity of a specific architecture based on the lower- and upper-bound of the number of regions and 2) To design an architecture based on the number of regions that is necessary for an specific application. A \emph{lower-bound} on the number of regions is described in the following theorem
\begin{theorem}
The maximal number of regions computed by a ReLU neural network, with $n_0$ inputs, $L$ hidden layers, and widths $n_l \geq n_0\ \forall l\in \{1,2,\hdots,L\}$, is lower-bounded by
\begin{align}\label{eq:lowerBound}
    \Bigg(\prod_{l=1}^{L-1}\floor{\frac{n_l}{n_0}}^{n_0}\Bigg)\sum_{j=0}^{n_0} \binom{n_L}{j}.
\end{align}
where $\floor{\cdot}$ is the floor function on fractions.
\end{theorem}
\begin{proof}
Proof can be found in \cite{Montufar2014} or \cite{pascanu}.
\end{proof}
From the hyperplane arrangement it can be shown that the maximal number of regions for any ReLU networks with a total of $N$ activation units is bounded from above by $2^N$ \cite{Montufar2014}. This bound is very loose, and not very useful. But there is also a tighter \emph{upper-bound} on the number of regions,  
\begin{theorem}
The maximal number of regions of a ReLU neural network, with $n_0$ inputs, $L$ hidden layers, and widths $n_l \geq n_0\ \forall l\in \{1,2,\hdots,L\}$, is upper-bounded by
\begin{align}\label{eq:upperBound}
    \sum_{(j_1,\hdots, j_L)\in J}\prod_{l=1}^{L} \binom{n_l}{j_l}
\end{align}
where $J = \{(j_1,\hdots, j_L)\in \mathbb{Z}^L:0 \leq j_l \leq \min \{n_0,n_1-j_1,\hdots,n_{l-1}-j_{l-1}, n_l\},\ \forall l\in[L]\}$
\end{theorem}
\begin{proof}
  See Theorem 1. in \cite{BoundingCounting}.
\end{proof}
These theoretical backgrounds give us better understanding of how a structure of neural network impacts its performance and also helps us to use some of these properties in order to construct the link with explicit MPC in the following sections.
\begin{figure}[!t]
    \centering
    \includegraphics[width=\linewidth]{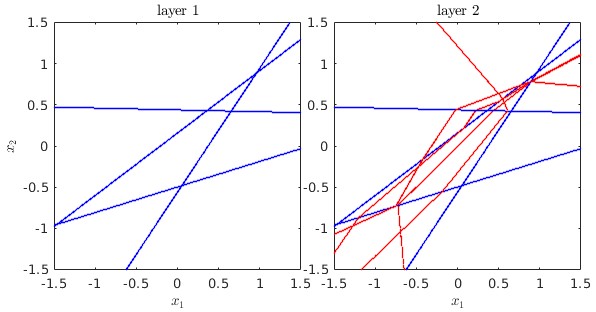}
    \caption{A ReLU DNN subdivides the input-space into polytopes. In fact, each hidden layer divides the input-space from the previous layer $h_{l-1}$, and this recursively subdivides the input-space of the whole network. Here we have a $2$--layer ReLU net with input $x\in \mathbb{R}^2$ and four activation units in each layer. The left plot shows the pre-activation functions $f_1 = W_1x+b_1$ that is equivalent to four hyperplanes in $\mathbb{R}^2$. Hidden units are activated in one side of their corresponding hyperplanes. The right plot shows both hyperplanes from the first and second layer in blue and red respectively. The hyperplanes in the second layer, as seen in the plot, are not straight lines, but they are bent at the first layer boundaries (blue lines). When those hyperplanes pass through different regions partitioned by the first layer, they will be bent. Therefore we have four activation boundary for four units in layer $2$, but they are not straight lines. In the right plot we can see all the regions that the network can partition on the input-space. Also it represents different affine functions over each polytope.}
    \label{fig:hyperplanes}
\end{figure}
\subsection{explicit MPC and PWA functions}
% In this section we present some preliminary definitions on explicit MPC, multiparametric programming, and piecewise affine functions and discuss some properties of these problems.
Given a dynamical system, the purpose of a constrained optimal control is to solve an optimization problem with a set of constraints on states $x_t$ and actions $u_t$ in order to find a sequence of actions $u^*_{0:\infty}$ that controls the system to a desired/reference state. We can formulate such problem as an infinite-horizon optimization problem  
\begin{equation}\label{eq:inf_hor}
	\begin{aligned}
		\hspace{0cm}
	J_{\infty}^*(x(0)) = 
	\min_{u_0,u_1,\ldots} & \sum\limits_{t=0}^{\infty} q (x_t, u_t) \\[1ex]
		\text{s.t.}\quad &  x_{t+1} = Ax_t + Bu_t, \\[1ex]
		& x_t\in \mathcal{X},\ u_t \in \mathcal{U}, \\[1ex]
		&  x_0 =  x(0), \\[1ex]
		& \forall t=0, 1, \hdots \ .
	\end{aligned}
\end{equation}
This problem \eqref{eq:inf_hor} cannot be solved easily due to its infinite horizon nature with constraints on states and actions \cite{borrelli2017predictive}; instead model predictive control (i.e. \emph{receding horizon control}) is a suitable approach to follow, which mimics \eqref{eq:inf_hor} by appropriate choice of $p(x_{N})$, $q(x_{k}, {u}_{k})$, and $\mathcal{X}_f$ as the following,
\begin{equation}\label{eq:MPC_eq} 
	\begin{aligned}
	J_{0}^*(x(t)) = 
		\min_{u_{0:(N-1)}} & p(x_{N}) + \sum_{k=0}^{N-1} q(x_{k}, {u}_{k}) \\
		\text{s.t}\quad &  x_{k+1} = Ax_{k} + Bu_{k},\\
		& x_k\in \mathcal{X},\ u_k \in \mathcal{U}, \\
		&x_{N} \in \mathcal{X}_f,\\
	   & x_{0} = x(t), \\
	   & \forall k = 0,\ldots,N-1.
	\end{aligned}
\end{equation}
Equation \eqref{eq:MPC_eq} can be seen as a multiparametric program (mp) in which $x(t)$ is the vector of parameters. In particular for the case of linear and quadratic cost functions with polyhedral constraints, it transpires that the solution to problem \eqref{eq:MPC_eq} is in fact a PWA function of the parameters $u^*(t)= f(x(t))$, an \emph{explicit} solution to the MPC controller. 
% \begin{defn}\label{def:pwa}
% A continuous piecewise affine (PWA) function $f:\Omega \subset \mathbb{R}^n \rightarrow \mathbb{R}^m$ can be written as a set of affine functions defined over $n_r$ polytopic regions $\Omega_i$
% \begin{align*}
%     f(x) = f_i(x) := W_ix+b_i, \quad x\in\Omega_i, \ \forall i=1,2,\hdots,n_r,
% \end{align*}
% where ${\bigcup}_{i=1}^N \Omega_i = \Omega$ and $\interior(\Omega_i) \cap \interior(\Omega_j) = \emptyset, \forall i\neq j$.
% \end{defn}

In a number of instances we may be interested in the constructing the PWA function corresponding to a ReLU net which is also the solution of a mp-LP/mp-QP problem. This may arise when, for example, we want to measure the suboptimality of a trained network with the solution of an explicit MPC. \emph{Inverse mp-LP/QP} studies this idea, constructing such optimization problems from PWA functions. The following theorem expresses this in detail,
% \begin{lemma}\label{lem:pwa}
% Every scalar-valued PWA function $f: \Omega \rightarrow \mathbb{R}$ can be written as the difference of two convex PWA functions $f(p) = \gamma(p) - \eta(p)$, where $\gamma(p):\Omega \rightarrow \mathbb{R}$ and $\eta(p):\Omega \rightarrow \mathbb{R}$ are defined over $r_{\gamma}$ and $r_{\eta}$ regions respectively.
% \end{lemma}
% \begin{proof}
% The proof can be found in \cite{twoconvexlemma} or \cite{Lygeros}. 
% \end{proof}
\begin{theorem}
Every continuous piecewise affine function $f: \mathbb{R}^{m}\rightarrow \mathbb{R}^n$ can be obtained as a linear map of the unique explicit solution $\hat{f}(x)$ of multi-parametric linear program in the form of
\begin{equation}\label{eq:mpp}
\begin{aligned}
    \hat{f}(x)\ \in\ & \argmin_z \ \ J(z,x) \\
    & \textnormal{s.t.} \quad (z,x) \in \Omega,
\end{aligned}
\end{equation}
with dimension $\hat{n}$, when $\hat{n} \leq 2n$.
\end{theorem}
The proof presented in \cite{Lygeros} is constructive, that means the proof establishes a procedure that results to the formulation of a mp-LP from a PWA function. The proof follows from the fact that every PWA function can be decomposed to two convex function and from there it is straightforward to construct a mp-LP for a convex PWA function. Note that, although the proof is constructive, it is still very hard (or even impossible) to implement it as an algorithm.

Now, referring to \figref{fig:big_picture} we can have a better understanding of the whole picture. Although it is possible to use learning to find an approximation of an explicit MPC policy, yet constructing a deep network from a PWA function needs to be studied.  
\section{Explicit MPC and ReLU DNN}
\label{secBOTH}
\begin{figure}
    \centering
    \includegraphics[width=\linewidth]{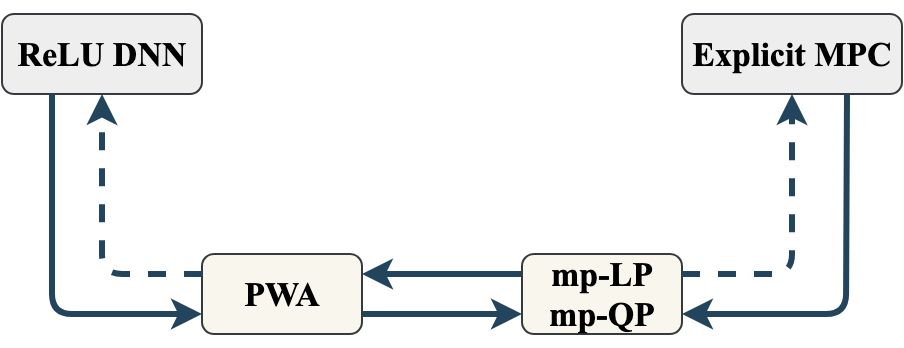}
    \caption{This illustration gives an entire perspective that this paper tries to depict. ReLU DNNs represent PWA function on polyhedra which subdivide the input-space. Assuming that the input to the neural network is the parameter $x(t)$, the network can exactly act as an explicit state feedback policy. The dashed arrows indicate the needs for further study of methods -analytical or approximate- which can reconstruct the mathematical structures of each block from the other in a constructive manner.}
    \label{fig:big_picture}
    % \emph{is there any way to reconstruct a ReLU DNN from a PWA function?}
    % And the right double dashed arrow denotes the question: \emph{is the explicit MPC formulation equivalent to the optimization problem derived based on the PWA function?}
\end{figure}
In this section we connect the ReLU DNN and explicit MPC through their underlying connection which is PWA functions. As mentioned in previous sections we know that every ReLU DNN has a continuous PWA function representation on the input-space, and vice versa (but not necessary in an explicit closed form, since constructing such a connection is not easy in general). 
% PWA function identification and equivalent representation has been subject of several works \cite{clustring, lygrosQP}.
\subsection{Identification of Input-Space in ReLU DNN}
In order to identify the different regions partitioned by ReLU NN on the input-space, we present an approximate method here that is an extension of the method introduced in \cite{Montufar2014}. We will show that it is possible to construct each pieces of a PWA function by extending the PWA representation of a shallow network (i.e. $L=1$).

Since every dimension of the output-space can be treated independently, here we assume the construction of a scalar-valued function $f: \Omega \subset \mathbb{R}^{n_0} \rightarrow \mathbb{R}$ from a DNN model (i.e the output-space is scalar $n_{L+1} = 1$), but as mentioned, the proposed method can be applied separately for each dimension in the case of vector-valued DNN models. Any scalar-valued affine function which is defined over its convex region $\Omega_i$ can be written as 
\begin{align}\label{eq:NNPWA}
    f_i(x) = u^\top x + c, \quad x \in \Omega_i,
\end{align}
where $u^\top \in \mathbb{R}^{n_0}$ and $c \in \mathbb{R}$. In order to construct $u^\top$ and $c$ in \eqref{eq:NNPWA}, we first consider a NN with one layer and then extend it to the deep nets.
\subsubsection{Shallow Network}
Note that we can reformulate a scalar rectifier function as  
\begin{align}
    \rect(h) = \mathbf{I}(h)\cdot h,
\end{align}
where $\mathbf{I}(h)$ is an indicator function defined as follows
\begin{align}
    \mathbf{I}(h) = 
    \begin{cases}
        1 \quad h > 0\\
        0 \quad \textnormal{otherwise}
    \end{cases}
\end{align}
Now considering a single layer NN $f: \mathbb{R}^{n_0}\rightarrow \mathbb{R}$, we rewrite it with the help of indicator function as
\begin{align}\label{eq:shallow}
    f(x) = W_2 \diag 
    \Bigg(
        \begin{bmatrix}
            \mathbf{I}(W_{1,1}x+b_{1,1}) \\
            \mathbf{I}(W_{1,2}x+b_{1,2}) \\
            \vdots\\
            \mathbf{I}(W_{1,n_1}x+b_{1,n_1}) \\
        \end{bmatrix}
    \Bigg) (W_{1}x+b_{1}).
\end{align}
Simplifying \eqref{eq:shallow}, $f(x)$ can be written more compactly as
\begin{align}\label{eq:shallow_compact}
    f(x) &= W_2\diag(\mathbf{I}_{f_1}(x)) W_1 x + W_2\diag(\mathbf{I}_{f_1}(x))b_1,
\end{align}
where $\diag(\mathbf{I}_{f_l}(x))$ is the compact form of indicator function for pre-activation $f_l$ in layer $l$. From \eqref{eq:shallow_compact} we can see that given input $x$ weight $u^\top$ and bias $c$ can be computed.
\subsubsection{Deep Network}
Now we can extend the derivation in \eqref{eq:shallow_compact} for deep network. Given an input $x$ from a region $\Omega_i$ we can construct the corresponding weight $u^\top$ and bias $c$ for each affine map $f_i$. The weight is computed by 
\begin{align}\label{eq:u}
    u^\top = W_{L+1} &\diag(\mathbf{I}_{f_{L}}(x)) W_{L}\ \hdots \nonumber\\ &\diag(\mathbf{I}_{f_{2}}(x)) W_{2}\ \diag(\mathbf{I}_{f_{1}}(x))W_1,    
\end{align}
A bias of the affine map $c$ also can be computed similarly
\begin{align}\label{eq:c}
    c &= W_{L+1} \diag(\mathbf{I}_{f_{L}}) W_{L}\ \hdots \ \diag(\mathbf{I}_{f_{2}}) W_{2}\ \diag(\mathbf{I}_{f_{1}}) b_1 \nonumber\\
    &+ W_{L+1} \diag(\mathbf{I}_{f_{L}}) W_{L}\ \hdots \diag(\mathbf{I}_{f_{2}}) b_2 \nonumber\\
    &+ \vdots \nonumber\\
    &+ W_{L+1} \diag(\mathbf{I}_{f_{L}}) b_{L}
\end{align}
Both equations \eqref{eq:u} and \eqref{eq:c} depend on input $x$, so we need to use a (large enough) set of samples from the input-space to be able to identify different affine responses of the output. 

It is worth mentioning that from \eqref{eq:u} and \eqref{eq:c} it is also possible to derive the corresponding affine function for each activation unit up to a specific hidden layer instead of the whole network. This means that any activation unit in any stage of a deep neural network can be written as a PWA function over the input-space of the network. This needs further study, but as a preliminary, we can ask "is there any connection between layers of a neural network and for example the horizon in model predictive control?"

\subsection{Learning DNN with Exact Architecture}
Several literature study the concept of learning approximate MPC through supervised or reinforcement learning process \cite{apprxMPC1, apprxMPC2}. But we can utilize the structure of PWA control policy to further improve the process of learning \cite{apprxMPC3}. The following theorem is the key concept in the process.  
\begin{theorem}\label{exactNN}
Any convex PWA function $f: \mathbb{R}^{n_0}\rightarrow \mathbb{R}$, which also can be formulated as pointwise maximum of $N$ affine functions $f(x) := \max_{i=1,\hdots,N}\ f_i(x)$, can be exactly presented by a ReLU DNN with width $n_l = n_0+1,\ \forall l\in [L]$ and depth $N$.
\end{theorem}
\begin{proof}    
See Theorem 2 in \cite{UniversalFunctionApprDNN}.
\end{proof}
Depending on the dimension of control input $u\in \mathbb{R}^m$, it may be needed to train up to $2m$ networks. In fact, every element in the control input vector can be treated separately. Each element $u^*(x(t)): \mathbb{R}^{n_0}\rightarrow \mathbb{R}$ is a PWA function which needs to be decomposed into the difference of two convex PWA functions. Finally, theorem \ref{exactNN} gives an exact design structure for each network. And presumably, this should result a better learning (smaller loss value, faster convergence, better accuracy), which ultimately impacts the performance of the controller that the trained network substitutes.
% In fact to the best of author's knowledge, there is no any constructive theorem that can explicitly link a PWA function with its DNN counterpart.
\section{Experimental validation}
\label{secExpValidation}
\begin{figure}
    \centering
    \includegraphics[width=\linewidth]{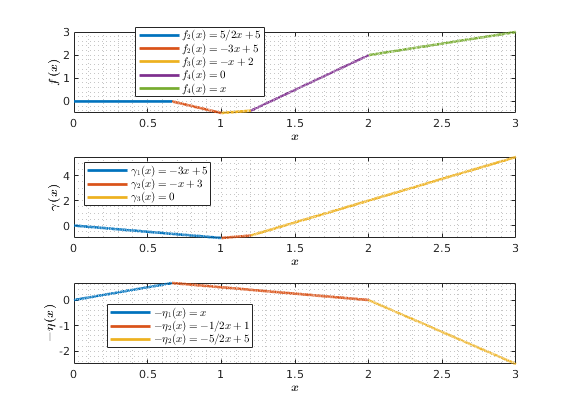}
    \caption{The plots show a scalar PWA function and its decomposition to two convex functions, $f(x) = \gamma(x) - \eta(x)$.}
    \label{fig:1d_exp}
\end{figure}
In this section we present a simple example (as a proof of concept) to illustrate the way to construct a mp-LP from a ReLU DNN. We examine a $2$-layer feedforward net with $n_1 = n_2 = 2$ (i.e. total of four activation units) and one dimensional input/output $x,\ y\in \mathbb{R}$. The two hidden layers are constructed as follows
\begin{align*}
    h_1 = \begin{bmatrix} h_{1,1} \\ h_{1,2}\end{bmatrix} 
        &= \max\{0, W_1x+b_1\} \\ 
        &= \max\bigg\{0, 
        \begin{bmatrix} -\nicefrac{3}{2}\\ 2\end{bmatrix} x + 
        \begin{bmatrix} 3\\ -2\end{bmatrix}
        \bigg\}
\end{align*}
\begin{align*}
    h_2 = \begin{bmatrix} h_{2,1} \\ h_{2,2}\end{bmatrix}
        &= \max\{0, W_2h_1+b_2\} \\
        &= \max\bigg\{0, 
        \begin{bmatrix} -1 & -1\\ \nicefrac{1}{2} & -1\end{bmatrix} h_1 + 
        \begin{bmatrix} 1\\ 2\end{bmatrix}
        \bigg\},
\end{align*}
and the linear map for the output layer is
\begin{align*}
    y = W_3h_2 = \begin{bmatrix} 1 & -1\end{bmatrix} h_2.
\end{align*}
The PWA function equivalent to the above feedforward network is
\begin{align}\label{eq:pwa_exp}
    y = f(x) = 
    \begin{cases}
        0 \quad &x \leq \frac{2}{3} \\
        -\frac{3}{2}x+1 \quad &\frac{2}{3} \leq x \leq 1\\
        \frac{1}{2}x-1 \quad &1 \leq x \leq \frac{6}{5}\\
        3x-4 \quad &\frac{6}{5} \leq x \leq 2\\
        x \quad &2 \leq x
    \end{cases}
\end{align}
since the PWA function \eqref{eq:pwa_exp} is not convex nor concave, we can decompose it into the difference of two convex functions $\gamma(x)$ and $\eta(x)$ as follows
\begin{align*}
    \gamma(x) &= 
    \begin{cases}
        -x & x \leq 1 \\
         x-2  & 1 \leq x \leq \frac{6}{5}\\
        \frac{7}{2}x-5 & \frac{6}{5} \leq x
    \end{cases}\\
    \eta(x) &= 
    \begin{cases}
        -x & x \leq \frac{2}{3}\\
        \frac{1}{2}x-1 & \frac{2}{3} \leq x \leq 2\\
        \frac{5}{2}x-5 & 2 \leq x
    \end{cases}
\end{align*}
Then we can construct the mp-LP counterpart that its solution is the same as PWA function \eqref{eq:pwa_exp}. Introducing decision variable $z\in \mathbb{R}^2$, we can write
\begin{align}\label{eq:mplp1d}
    \begin{split}
    J^*(x) = \min_{z\in \mathbb{R}^2} \ \  
        \begin{bmatrix}1 & -1\end{bmatrix}
        \begin{bmatrix}z_1\\ z_2\end{bmatrix}\\
        \text{s.t. }  -x \leq z_1 &\quad x \geq z_2 \\
                      x-2 \leq z_1 &\quad -\frac{1}{2}x+1 \geq z_2 \\
                      \frac{7}{2}x-5 \leq z_1 &\quad -\frac{5}{2}x+5 \geq z_2, \\
                      x \in [0, 3] &
    \end{split}
\end{align}
and then we can construct $f(x)$ with linear map $T$ as
\begin{align*}
    f(x) = T\hat{f}(x) = \begin{bmatrix}1 & 1\end{bmatrix}\hat{f}(x)
\end{align*}
In fact the explicit solutions to the mp-LP are $z_1^* = \gamma(x)$, $z_2^* = -\eta(x)$ and $\hat{f}(x) = \begin{bmatrix}\gamma(x) & -\eta(x)\end{bmatrix}^T$, which exactly follows the constructive proof in \cite{Lygeros}.
\begin{figure}
    \centering
    \includegraphics[width=\linewidth]{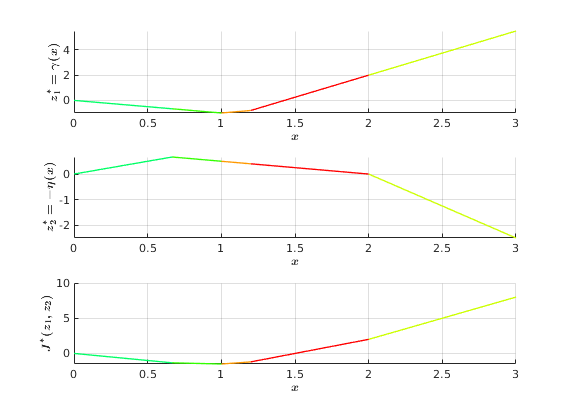}
    \caption{The plots show the explicit solution to the mp-LP \eqref{eq:mplp1d} using MPT3 toolbox \cite{MPT3}. \textbf{Top \& center:} show plots of the optimizers $z_1^*=\gamma(x)$ and $z_2^*=-\eta(x)$ which both are functions of the parameter $x$. the solution exactly results the PWA function \eqref{eq:pwa_exp}. \textbf{bottom:} The plot depicts the optimal value $J^*(x)$ which is a function of parameter $x$ and also is convex and PWA as we know from the theory of mp-LP (corollary $11.5$ in \cite{borrelli2017predictive}).}
    \label{fig:1d_exp2}
\end{figure}

\section{Conclusion}
\label{secConclusion}
We presented an overview of ReLU deep neural networks; and discussed several structural properties of such models which are key concepts of using a ReLU network as an explicit state feedback policy for a model predictive controller. Specifically, we argued since any ReLU network models a PWA function on polyhedra, it would be a perfect choice to use a ReLU network instead of state feedback policy computed by an explicit MPC procedure considering storage and execution complexity of such controllers in real-time. We also presented a sample-based method that identifies different affine pieces of a ReLU networks. For future work, alongside further development of some initial findings in this paper, other very recently new ideas such as representing ReLU DNN as a mixed-integer linear problem \cite{BoundingCounting} can be the subject of further investigation. 

% We presented an overview of ReLU deep neural networks; and discussed several structural properties of such models which are key concepts of using a ReLU network as an explicit policy of a model predictive controller. Specifically, we argued since any ReLU network models a PWA function on polyhedra, it would be a perfect choice to use a ReLU network instead of state feedback policy computed by an explicit MPC procedure considering storage and execution complexity of such controllers in real-time. We also presented a sample-based method that identifies different affine pieces of a ReLU networks. For future work, alongside further development of some initial findings in this paper, other, very recently, new ideas such as representing ReLU DNN as a mixed-integer linear problem \cite{BoundingCounting} can be the subject of further investigation.
\section{Acknowledgement}
\label{secAcknowledgement}
This work was partially supported by the NSF Graduate Research Fellowship under Grant No. DGE 1106400 awarded to the first author, and the Cheryl and John Neerhout, Jr. Distinguished chair fund.

% This work was supported in part by the National Science Foundation Graduate Research Fellowship under Grant No. DGE 1106400.

\section*{References}
{%\footnotesize 
\printbibliography[heading=none, resetnumbers=true]
}

\end{document}